\documentclass[11pt]{article}

\usepackage{color}
\usepackage{amsmath}
\usepackage{amsthm}
\usepackage{amsfonts}
\usepackage[round]{natbib}
\usepackage[margin=1.1in]{geometry}
\usepackage{kpfonts}
\newtheorem{theorem}{Theorem}[section]

\newtheorem{assumption}{Assumption}[section]
\newtheorem{proposition}{Proposition}[section]

\theoremstyle{definition}
\newtheorem{definition}{Definition}[section]
\newtheorem{example}{Example}[section]

\def\paren#1{\left( #1 \right)}    
\def\brack#1{\left[ #1 \right]}     
\newcommand\ex{\ensuremath{\mathbb{E}}} 
\newcommand{\1}[1]{\ensuremath{\mathbb{I}} \left[#1\right]} 
\newcommand{\X}[0]{\mathcal{X}}
\newcommand{\A}[0]{\mathcal{A}}
\newcommand{\Y}[0]{\mathcal{Y}}

\newcommand{\reals}{{\mbox{\bf R}}}
\newcommand{\opt}{^{*}}
\newcommand{\eps}{\varepsilon}
\DeclareMathOperator*{\argmax}{arg\,max}

\newcommand{\indep}{\perp \! \! \! \perp}
\newcommand{\simiid}{\overset{\textrm{i.i.d.}}{\sim}}
\def\norm#1{\left\| #1 \right\|}
\def\set#1{\left\{ #1 \right\}}

\DeclareMathOperator*{\sign}{sign} 

\newcommand{\gi}{\mathrm{{\tt GoodIncentives}}}
\newcommand{\gc}{\mathrm{{\tt OutcomeMonotonicCost}}}

\newcommand{\improve}{I}

\newcommand{\parents}{\mathrm{\textbf{PA}}}
\newcommand{\ancestors}{\mathrm{\textbf{A}}}

\newcommand{\fail}{{\tt Fail}}

\usepackage{mathtools}
\mathtoolsset{showonlyrefs}

\usepackage{etoolbox, xspace}

\gdef\isarxiv{1}
\newcommand{\arxiv}[2]{\ifdefined\isarxiv{#1}\else{#2} \fi }

\begin{document}

\title{Strategic Classification is Causal Modeling in Disguise}

\author{John Miller\and Smitha Milli\and Moritz Hardt}

\maketitle

\begin{abstract}
Consequential decision-making incentivizes individuals to strategically adapt
their behavior to the specifics of the decision rule. While a long line of
work has viewed strategic adaptation as gaming and attempted to mitigate its
effects, recent work has instead sought to design classifiers that incentivize
individuals to improve a desired quality. Key to both accounts is a cost
function that dictates which adaptations are rational to undertake. In this
work, we develop a causal framework for strategic adaptation. Our causal
perspective clearly distinguishes between gaming and improvement and reveals
an important obstacle to incentive design. We prove any procedure for
designing classifiers that incentivize improvement must inevitably solve a
non-trivial causal inference problem. Moreover, we show a similar result
holds for designing cost functions that satisfy the requirements of previous work.
With the benefit of hindsight, our results show much of the prior work on
strategic classification is causal modeling in disguise.
\end{abstract}

\section{Introduction}
\label{sec:intro}
Individuals faced with consequential decisions about them often use knowledge of
the decision rule to strategically adapt towards achieving a desirable outcome.
Much work in computer science views such \emph{strategic adaptation} as
adversarial behavior~\citep{dalvi2004adversarial,bruckner2012static},
manipulation, or \emph{gaming}~\citep{hardt2016strategic,dong2018strategic}.  More recent work
rightfully recognizes that adaptation can also correspond to attempts at
self-improvement~\citep{bambauer2018algorithm,kleinberg2019classifiers}. 
Rather than seek classifiers that are robust to
gaming~\citep{hardt2016strategic,dong2018strategic}, these works suggest to
design classifiers that explicitly \emph{incentive improvement} on some target
measure~\citep{kleinberg2019classifiers,alon2019multiagent,khajehnejad2019optimal,haghtalab2020maximizing}.

Incentivizing improvement requires a clear distinction between gaming and
improvement. While this distinction may be intuitive in some cases, in others,
it is subtle. Do employer rewards for punctuality improve productivity? It sounds
plausible, but empirical evidence suggests
otherwise~\citep{gubler2016motivational}.
Indeed, the literature is replete with examples of failed incentive
schemes~\citep{oates2015window, rich1984some,belot2016spillover}.

Our contributions in this work are two-fold. First, we provide the missing
formal distinction between gaming and improvement. This distinction is a
corollary of a comprehensive causal framework for strategic adaptation that we
develop. Second, we give a formal reason why incentive design is so difficult.
Specifically, we prove any successful attempt to incentivize improvement must
have solved a non-trivial causal inference problem along the way.

\subsection{Causal Framework}

We conceptualize individual adaptation as performing an \emph{intervention} in a
causal model that includes all relevant features~$X$, a predictor~$\hat Y$, as
well as the target variable~$Y$. We then characterize gaming and improvement by
reasoning about how the corresponding intervention affects the predictor
$\hat{Y}$ and the target variable $Y$. This is illustrated in
Figure~\ref{fig:figure1}.

We combine the causal model with an \emph{agent-model} that describes how
individuals with a given setting of features respond to a classification rule.
For example, it is common in strategic classification to model agents as being
rational with respect to a \emph{cost function} that quantifies the cost of
feature changes.

Combining the causal model and agent model, we can separate improvement from
gaming. Informally speaking, improvement corresponds to the case where the
agent response to the predictor causes a positive change in the target
variable~$Y$. Gaming corresponds to the case where the agent response causes a
change in the prediction~$\hat Y$ but not the underlying target variable~$Y$.
Making this intuition precise, however, requires the language of counterfactuals
of the form: What value would the variable~$Y$ have taken had the individual
changed her features to $X'$ given that her original features were $X$?

If we think of the predictor as a \emph{treatment}, we can analogize our notion
of improvement with the established causal quantity known as \emph{effect of
treatment on the treated}.

\subsection{Inevitability of Causal Analysis}

Viewed through this causal lens, only adaptations on causal variables can lead
to improvement. Intuitively, any mechanism for incentivizing improvement must
therefore capture some knowledge of the causal relationship between the features
and the target measure. We formalize this intuition and prove causal modeling is
unavoidable in incentive design. Specifically, we establish a computationally
efficient reduction from discovering the causal structure relating the variables
(sometimes called causal graph discovery) to a sequence of incentive design
problems. In other words, designing classifiers to incentivize improvement is as
hard as causal discovery.

Beyond incentivizing improvement, a number of recent works model individuals as
acting in accordance with well-behaved cost functions that capture the
difficulty of changing the target variable. We show constructing such
\emph{outcome-monotonic} cost functions also requires modeling the causal
structure relating the variables, and we give a similar reduction from designing
outcome-monotonic cost functions to causal discovery.

In conclusion, our contributions show that---with the benefit of
hindsight---much work on strategic classification turns out to be causal
modeling in disguise.

\subsection{Related Work}
This distinction between causal and non-causal manipulation in a classification
setting is intuitive, and such considerations were present in early work on
statistical risk assessment in lending~\citep{hand1997graphical}.  Although they
do not explicitly use the language of causality, legal
scholars~\citet{bambauer2018algorithm} give a qualitatively equivalent
distinction between gaming and improvement. While we focus on the incentives
classification creates for individuals, \citet{everitt2019understanding}
introduce a causal framework to study the incentives classification creates for
decision-makers, e.g. which features the decision-maker is incentivized to use.

\begin{figure}[t!]
    \centering
    \arxiv{
        \includegraphics[height=0.35\textheight]{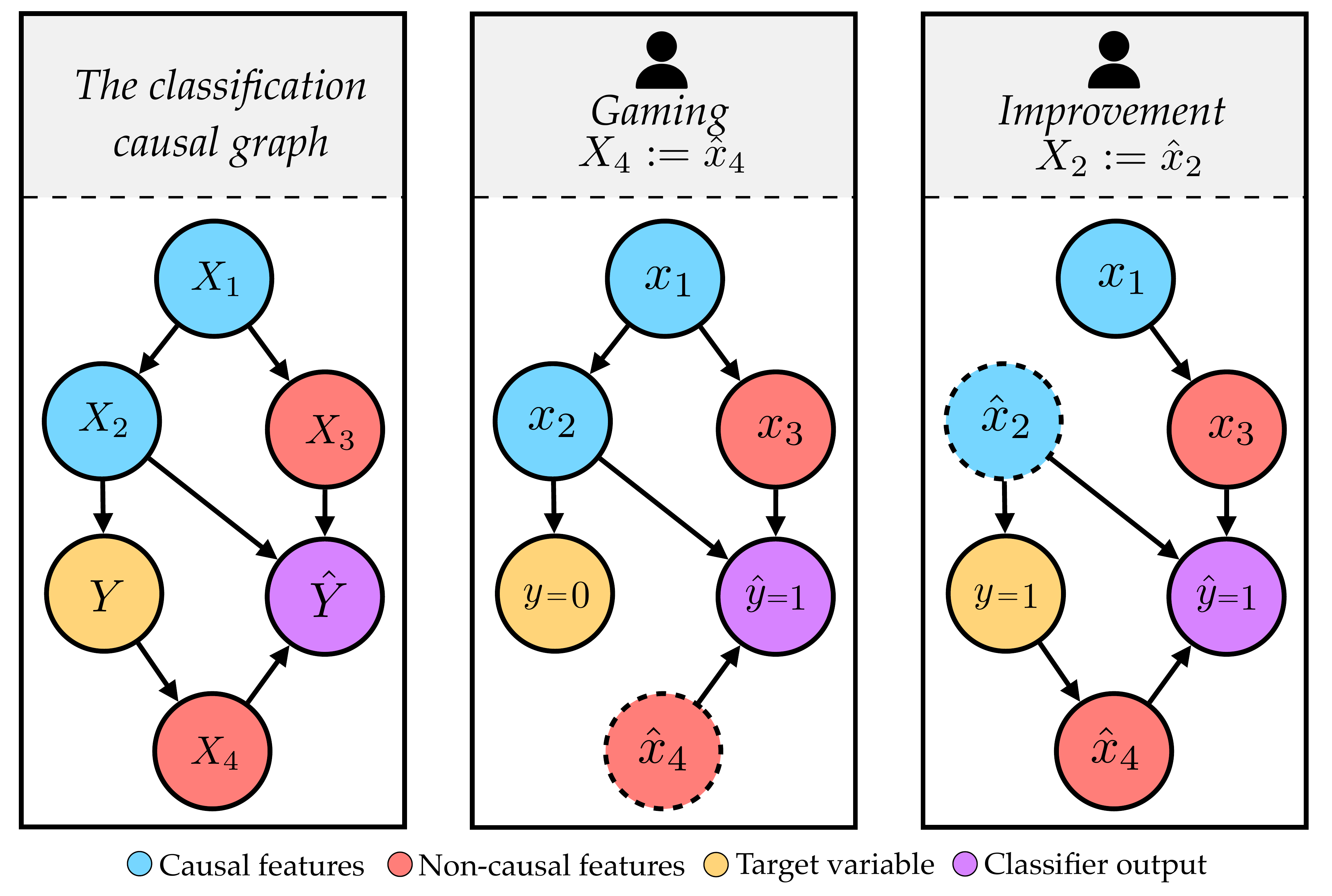}
     }{
        \includegraphics[width=\linewidth]{figs/figure1}
    }
    \caption{Illustration of the causal framework for strategic adaptation.
    Adaptation is modeled as interventions in a \emph{counterfactual} causal
    graph, conditioned on the individual's initial features $X$.  Gaming
    corresponds to interventions that change the classification $\hat{Y}$, but
    do not change the true label $Y$.  Improvement corresponds to
    interventions that change both the classification $\hat{Y}$ and the true
    label $Y$. Incentivizing improvement requires inducing agents to intervene
    on \emph{causal} features that can change the label $Y$ rather than
    \emph{non-causal} features. Distinguishing between these two categories of
    features in general requires causal analysis.}
    \label{fig:figure1}
\end{figure}

Numerous papers in strategic
classification~\citep{bruckner2012static,dalvi2004adversarial,hardt2016strategic,dong2018strategic}
focuses on game-theoretic frameworks for preventing gaming. These frameworks
form the basis of our agent-model,
and~\citet{milli2019social,garg2019role,khajehnejad2019optimal} introduce the
outcome-monotonic cost functions we analyze in Section~\ref{sec:cost_reduction}.
Since these approaches do not typically distinguish between gaming and
improvement, the resulting classifiers can be unduly conservative, which in turn
can lead to undesirable social
costs~\citep{hu2019disparate,milli2019social,garg2019role}. 

The creation of decision rules with optimal incentives, including incentives for
improvement, has been long studied in economics, notably in principle-agent
games \citep{ross1973economic, grossman1992analysis}. In machine learning,
recent work by \citet{kleinberg2019classifiers} and~\citet{alon2019multiagent}
studies the problem of producing a classifier that incentivizes a given ``effort
profile'', the amount of desired effort an individual puts into certain actions,
and assumes the evaluator knows which forms of agent effort would lead to
improvement, which is itself a form of causal knowledge.
\citet{haghtalab2020maximizing} seek to design classifiers that maximize
improvement across the population, while~\citet{khajehnejad2019optimal} seek to
maximize institutional utility, taking into account both improvement and gaming.
While these works do not use the language of causality, we demonstrate that
these approaches nonetheless must perform some sort of causal modeling if they
succeed in incentivizing improvement.

In this paper, we primarily consider questions of improvement or gaming from the
perspective of the decision maker. However, what gets categorized as
improvement or gaming also often reflects a moral judgement---gaming is
bad, but improvement is good. Usually good or bad means good or bad from the
perspective of the system operator.  \citet{ziewitz2019rethinking} analyzes how
adaptation comes to be seen as ethical or unethical through a case study
on search engine optimization. \citet{burrell2019when} argue that gaming can
also be a form of individual ``control'' over the decision rule and that the
exercise of control can be legitimate independently of whether an action is
considered gaming or improvement in our framework.

\section{Causal background} 
\label{sec:background}
We use the language of \emph{structural causal
models}~\citep{pearl2009causality} as a formal framework for causality. A
structural causal model (SCM) consists of endogenous variables $X = (X_1, \dots,
X_n)$, exogenous variables $U = (U_1, \dots, U_n)$, a distribution over the
exogenous variables, and a set of structural equations that determine the values
of the endogenous variables. The structural equations can be written
\begin{align*}
    X_i = g_i(\parents_i, U_i), \quad i = 1, \dots, n \,,
\end{align*}
where $g_i$ is an arbitrary function, $\parents_i$ represents the other
endogenous variables that determine $X_i$, and $U_i$ represents exogenous noise
due to unmodeled factors.

A structural causal model gives rise to a \emph{causal graph} where a
directed edge exists from $X_i$ to $X_j$ if $X_i$ is an input to the structural
equation governing $X_j$, i.e. $X_i \in \parents_j$. We restrict
ourselves to \emph{Markovian} structural causal models, which have an acyclic
causal graph and independent exogenous variables. The \emph{skeleton} of a
causal graph is the undirected version of the graph.

An \emph{intervention} is a modification to the structural equations of an SCM.
For example, an intervention may consist of replacing the structural equation
$X_i = g_i(\parents_i, U_i)$ with a new structural equation $X_i \coloneqq x_i$
that holds $X_i$ at a fixed value. We use $\coloneqq$ to denote modifications of
the original structural equations. When the structural equation for one variable
is changed, other variables can also change. Suppose $Z$ and $X$ are two
endogenous nodes, Then, we use the notation $Z_{X \coloneqq x}$ to refer to the
variable $Z$ in the modified SCM with structural equation $X \coloneqq x$.

Given the values $u$ of the exogenous variables $U$, the endogenous variables
are completely deterministic. We use the notation $Z(u)$ to represent the
deterministic value of the endogenous variable when the exogenous variables $U$
are equal to $u$. Similarly, $Z_{X \coloneqq x}(u)$ is the value of $Z$ in the
modified SCM with structural equation $X \coloneqq x$ when $U=u$. 

More generally, given some event $E$, $Z_{X \coloneqq x}(E)$ is the random
variable $Z$ in the modified SCM with structural equations $X \coloneqq x$ where
the distribution of exogenous variables $U$ is updated by conditioning on the
event $E$. We make heavy use of this \emph{counterfactual} notion. For more
details, see~\citet{pearl2009causality}.

\section{A Causal Framework for Strategic Adaptation}
\label{sec:framework}
In this section, we put forth a causal framework for reasoning about the
incentives induced by a decision rule. Our framework consists of two components:
\emph{the agent model} and \emph{the causal model}. The agent model is a
standard component of work on strategic classification and determines
what actions agents undertake in response to the decision rule. The causal model 
enables us to reason cogently about how these actions affect the agent's
true label. Pairing these models together allow us to distinguish between
incentivizing \emph{gaming} and incentivizing \emph{improvement}.

\subsection{The Agent Model}
As a running example, consider a software company that uses a classifier to
filter software engineering job applicants. Suppose the model considers, among
other factors, open-source contributions made by the candidate.  Some
individuals realize this and \emph{adapt}---perhaps they polish their resume;
perhaps they focus more of their energy on making open source contributions.
The agent model describes precisely how individuals choose to adapt in response
to a classifier.

As in prior work on strategic
classification~\citep{hardt2016strategic,dong2018strategic}, we model
individuals as \emph{best-responding} to the classifier. Formally, consider an
individual with features $x \in \X \subseteq \mathbb{R}^n$, label $y \in \Y
\subseteq \mathbb{R}$, and a classifier $f: \mathbb{R}^n \to \Y$. 
The individual has a set of available actions $\A$, and, in response to the
classifier $f$, takes action $a \in \A$ to adapt her features from $x$ to 
$x+ a$. For instance, the features $x$ might encode the candidate's existing
open-source contributions, and the action $a$ might correspond to making
additional open-source contributions. Crucially, these modifications incur a
\emph{cost} $c(a; x)$, and the action the agent takes is determined by directly
balancing the benefits of classification $f(x+a)$ with the cost of adaptation
$c(a; x)$. 

\begin{definition}[Best-response agent model]
Given a cost function $c: \A \times \X \to \reals_{+}$ and a classifier $f:\X
\to \Y$, an individual with features $x$ best responds to the classifier $f$
by choosing action
\begin{align}
    a\opt \in \argmax_{a \in \A} f(x + a) - c(a; x).
\end{align}
    Let $\Delta(x; f) = x + a\opt$ denote a \emph{best-response} of the agent
    to classifier $f$.  When clear from context, we omit the dependence on $f$
    and write $\Delta(x)$.
\end{definition}

In the best-response agent model, the cost function completely dictates what
actions are rational for the agent to undertake and occupies a central modeling
challenge. We discuss this further in Section~\ref{sec:cost_reduction}. Our
definition of the cost function in terms of an action set $\A$ is motivated by
~\citet{ustun2019actionable}. However, this formulation is completely equivalent
to the agent-models considered in other
work~\citep{hardt2016strategic,dong2018strategic}. In contrast to prior work,
our main results only require that individuals approximately best-respond to the
classifier. 

\begin{definition}[Approximate best-response]
    For any $\eps \in (0, 1)$, say $\Delta_\eps(x, f) = x + \tilde{a}$ is an
    $\eps$- best-response to classifier $f$ if $f(x + \tilde{a}) - c(\tilde{a}; x) \geq
    \eps \cdot (\max_{a} f(x + a) - c(a; x))$.
\end{definition}

\subsection{The Causal Model}
While the agent model specifies which actions the agent takes in response to the
classifier, the causal model describes how these actions effect the individual's true label.

Returning to the hiring example, suppose individuals decide increase their
open-source contributions, $X$. Does this improve their software engineering
skill, $Y$?  There are two different causal graphs that explain this scenario.
In one scenario, $Y \rightarrow X$: the more skilled one becomes, the more
likely one is to contribute to open-source projects. In the other scenario, $X
\rightarrow Y$: the more someone contributes to open source, the more skilled
they become.  Only in the second world, when $X \rightarrow Y$, do adaptations
that increase open-source contributions raise the candidate's skill.  

More formally, recall that a structural causal model has two types of nodes:
endogenous nodes and exogenous nodes. In our model, the endogenous nodes are the
individual's true label $Y$, their features $X = \{X_1, \dots, X_n\}$, and their
classification outcome $\hat{Y}$.  The structural equation for $\hat{Y}$ is
represented by the classifier $\hat{Y} = f(Z)$, where $Z \subseteq X$ are the
features that the classifier $f$ has access to and uses.  The exogenous
variables $U$ represent all the other unmodeled factors. 

For an individual with features $X = x$, let $\Delta(x, f)$ denote the agent's
response to classifier $f$. Since the agent chooses $\Delta(x, f)$ as a function
of the observed features $x$, the label after adaptation is a
\emph{counterfactual} quantity. This, we model the individual's adaptation as an
intervention in the submodel \emph{conditioned on observing features} $X=x$.
What value would the label $Y$ take if the individual had features $\Delta(X,
f)$, given that her features were originally $X$? 

Formally, let $A = \{i: \Delta(x, f)_i \neq x_i\}$ be the subset of features the
individual adapts, and let $X_A$ index those features. Then, the label after
adaptation is given by $Y_{X_A \coloneqq \Delta(x, f)_A}(\set{X = x})$. The
dependence on $A$ ensures that, if an individual only intervenes on a subset of
features, the remaining features are still consistent with the original causal
model. For brevity, we omit reference to $A$ and write $Y_{X\coloneqq\Delta(x,
f)}(\set{X=x})$.  In the language of potential outcomes, both $X$ and $Y$ are completely
deterministic given the exogenous variables $U=u$, and we can express the label
under adaptation as $Y_{X \coloneqq \Delta(x, f)}(u)$. 

Much of the prior literature in strategic classification eschews explicit causal
terminology and instead posits the existence of a ``qualification function'' or a
``true binary classifier'' $h: \X \to \Y$ that maps the individual's features to their ``true
quality''~\citep{hardt2016strategic,hu2019disparate,garg2019role,haghtalab2020maximizing}.
Such a qualification function should be thought of as the strongest possible
causal model, where $X$ is causal for $Y$, and the structural equation
determining $Y$ is completely deterministic.

\subsection{Evaluating Incentives}
Equipped with both the agent model and the causal model, we can formally
characterize the incentives induced by a decision rule $f$. Key to our
categorization is the notion of \emph{improvement}, which captures how 
the classifier induces agents to change their label on average over the
population baseline.

\begin{definition}
    \label{def:improvement}
    For a classifier $f$ and a distribution over features $X$ and label $Y$
    generated by a structural causal model, define the \emph{improvement}
    incentivized by $f$, as
    \begin{align}
        \improve(f) = 
        \ex_X \ex \brack{Y_{X \coloneqq \Delta(x, f)}(\set{X=x})} - \ex \brack{Y}.
    \end{align}
    If $\improve(f) > 0$, we say that $f$ \emph{incentivizes improvement}. Otherwise,
    we say that $f$ \emph{incentivizes gaming}.
\end{definition}

By the tower property, definition~\ref{def:improvement} can be equivalently
written in terms of potential outcomes $\improve(f) = \ex_U \brack{Y_{X
\coloneqq \Delta(x, f)}(U) - Y(U)}$. In this view, if we imagine exposure to the
classifier $f$ as a treatment, then improvement is the \emph{treatment effect of
exposure to classifier $f$ on the label $Y$}. In general, since all
individuals are exposed and adapt to the classifier in our model, and
estimating improvement becomes an exercise in estimating the effect of treatment
on the treated, and identifying assumptions are provided
in~\citet{shpitser2009effects}. Our notion of improvement is closely related to
notion of ``gain'' discussed in~\citet{haghtalab2020maximizing}, albeit with a
causal interpretation. We can similarly characterize improvement at the level of
the individuals.

\begin{definition}
    \label{def:ind_improve}
    For a classifier $f$ and a distribution over features $X$ and label $Y$
    generated by a structural causal model, define the \emph{improvement}
    incentivized by $f$ for an individual with features $x$ as
	\begin{align*}
        I(f; x) = \ex \brack{Y_{X \coloneqq \Delta(x, f)}(\set{X=x})} - \ex \brack{Y \mid X=x}.
	\end{align*}
\end{definition}

At first glance, the causal model and Definition~\ref{def:improvement} appear
to offer a convenient heuristic for determining whether a classifier incentivizes
gaming. Namely, does the classifier rely on non-causal features?
However, even a classifier that uses purely non-causal features can still
incentivize improvement if manipulating upstream, causal features is less costly
than directly manipulating the non-causal features. The following example
formalizes this intuition. Thus, reasoning about improvement requires
considering both the agent model and the causal model.

\begin{example}
    Suppose we have a structural causal model with features $X, Z$ and label $Y$
    distributed as $X \coloneqq U_X$, $Y \coloneqq X + U_Y$, and $Z \coloneqq Y
    + U_Z$, where $U_X, U_Y, U_Z \simiid \mathcal{N}(0, 1)$.  Let the classifier
    $f$ depend only on the non-causal feature, $Z$, $f(z) = \hat{y}$. Let $\A =
    \mathbb{R}^2$, and define the cost function $c(a; x) = (1/2)a^\top C a$,
    where $C \succ 0$ is a symmetric, positive definite matrix with $\det(C) =
    1$. Then, direct computation shows $\Delta(x, z; f) = (x - C_{12}, z +
    C_{11})$, and $I(f) = -C_{12}$. Hence, provided $C_{12} < 0$, $f$
    incentivizes improvement despite only rely on non-causal features.  When
    $C_{12} < 0$ changing $x$ and $z$ jointly is less costly than manipulating
    $z$ alone. This \emph{complementarity}~\citep{holmstrom1991multitask} allows the
    decision-maker to incentivize improvement using only a non-causal feature.
    This example is illustrated in Figure~\ref{fig:agent_model}.
\end{example}

\begin{figure}
	\centering
    \includegraphics[scale=0.4]{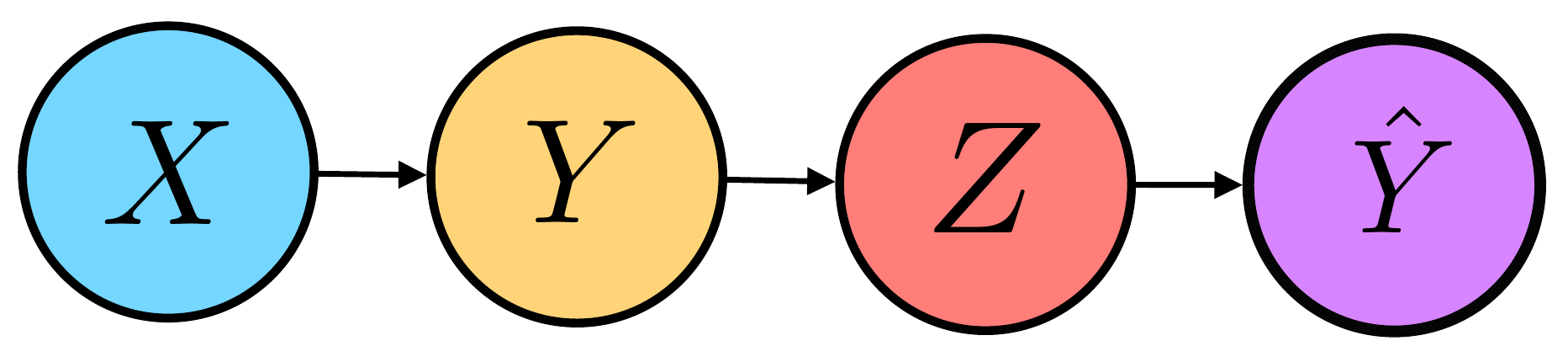}
    \caption{Reasoning about incentives requires both the agent-model and the
    causal model. The cost function plays a central role in the agent-model.
    Even though the classification $\hat{Y}$ only depends on the non-causal
    feature $Z$, the agent can change the label by manipulating, $X$, $Z$ or
    both, depending on the cost function. The causal model determines how the
    agent's adaptation affects the target measure, but the agent model, and in
    turn the cost function, determines which actions the agent actually takes.}
    \label{fig:agent_model}
\end{figure}

\section{Incentivizing Improvement Requires Causal Modeling}
\label{sec:good_incentives}
Beyond evaluating the incentives of a particular classifier, recent work has
sought to \emph{design} classifiers that explicitly incentivize
improvement. \citet{haghtalab2020maximizing} seeks classifiers that
\emph{maximize} the improvement of strategic individuals according to some
quality score. Similarly, both~\citet{kleinberg2019classifiers}
and~\citet{alon2019multiagent} construct decision-rules that incentivize 
investment in a desired ``effort profile'' that ultimately leads to individual
improvement. In this section, we show that when these approaches succeed in
incentivizing improvement, they must also solve a non-trivial causal modeling
problem. Therefore, while they may not explicitly discuss causality, much of this
work is \emph{necessarily} performing causal reasoning.

\subsection{The Good Incentives Problem}
We first formally state the problem of designing classifiers that incentivize
improvement, which we call the \emph{good incentives problem}. Consider the
hiring example presented in Section~\ref{sec:framework}. A decision-maker has
access to a distribution over features (open-source contributions, employment
history, coding test scores, etc), a label (engineering ability), and wishes to
design a decision rule that incentivizes strategic individuals to improve their
engineering ability.  As discussed in Section~\ref{sec:framework}, the
decision-maker must reason about the agent model governing adaptation, and we
assume agent's \emph{approximately} best-respond according to some specified
cost function.

\begin{definition}[Good Incentives Problem]
    Assume agents $\eps$-best-respond to the classifier for some $\eps>0$.
    Given:
    \begin{enumerate}
    \item A joint distribution $P_{X, Y}$ over examples $(x, y) \in \X \times \Y$  entailed by
    structural causal model, and
    \item A cost function  $c\colon \A \times \X \to \mathbb{R}_{+}$,
    \end{enumerate}
    Find a classifier $f\opt\colon \X \to \Y$ that incentivizes improvement, i.e.
    find a classifier with $\improve(f\opt) > 0$. If no such classifier exists,
    output $\fail$.
\end{definition}

The good incentives problem is closely related to the improvement problem
studied in~\citet{haghtalab2020maximizing}. Translated into our framework,
\citet{haghtalab2020maximizing} seek classifiers that optimally incentivize
improvement and solve $\max_{f} \improve(f)$, which is a more difficult problem
than finding \emph{some} classier that leads to improvement.  

In the sequel, let $\gi$ be an oracle for the Good Incentives problem.
$\gi$ takes as input a cost function and a joint distribution over features and
label, and either returns a classifier that incentivizes improvements or returns
no such classifier exists.

\subsection{A Reduction From Causal Modeling to Designing Good Incentives}
Incentivizing improvement requires both (1) knowing which actions lead to
improvement, and (2) incentivizing individuals to take those actions. Since only
adaptation of causal features can affect the true label $Y$, determining which
actions lead to improvement necessitates distinguishing between causal and
non-causal features. Consequently, any procedure that can provide incentives for
improvement must capture some, possibly implicit, knowledge about the causal
relationship between the features and the label.

The main result of this section generalizes this intuition and establishes a
reduction from orienting the edges in a causal graph to designing classifiers
that incentivize improvement. Orienting the edges in a causal graph is not
generally possible from observational data alone \citep{peters2017elements},
though it can be addressed through active intervention
\citep{eberhardt2005number}. Therefore, any procedure for constructing
classifiers that incentivize improvement must at its core also solve a
non-trivial causal discovery problem.  

We prove this result under a natural assumption: improvement is always possible
by manipulating causal features. In particular, for any edge $V \to W$ in the
causal graph, there is always \emph{some} intervention on $V$ a strategic agent
can take to improve $W$. We formally state this assumption below, and, as a
corollary, we prove this assumption holds in a broad family of causal graphs:
additive noise models.

\begin{assumption}
    \label{assump:control}
    Let $G = (X, E)$ be a causal graph, let $X_{-W}$ denote the random
    variables $X$ excluding node $W$. For any edge $(V, W) \in E$ with $V \to W$,
    there exists a real-valued function $h$ mapping $X_{-w}$ to an intervention
    $v\opt = h(x_{-w})$ so that
    \begin{align}
        \ex_{X_{-W}} \ex\brack{W_{V \coloneqq h(x_{-w})}\paren{\set{X_{-W} = x_{-w}}}} 
        > \ex \brack{W}.~\label{eq:control}
    \end{align}
\end{assumption}
Importantly, the intervention $v\opt= h(x_{-w})$ discussed in
Assumption~\ref{assump:control} is an intervention in the counterfactual
model, conditional on observing $X_{-W} = x_{-w}$. In strategic
classification, this corresponds to choosing the adaptation conditional on the
values of the observed features. Before proving Assumption~\ref{assump:control}
holds for faithful additive noise models, we first state and prove the main
result.

Under Assumption~\ref{assump:control}, we exhibit a reduction from orienting the
edges in a causal graph to the good incentives problem. While
Assumption~\ref{assump:control} requires Equation~\eqref{eq:control} to hold for
every edge in the causal graph, it is straightforward to modify the result when
Equation~\eqref{eq:control} only holds for a subset of the edges.

\begin{theorem}
    \label{thm:causal_reduction}
    Let $G = (X, E)$ be a causal graph induced by a structural causal model that
    satisfies Assumption~\ref{assump:control}. Assume $X$ has bounded support
    $\X$. Given the skeleton of $G$, using $|E|$ calls to $\gi$, we can orient
    all of the edges in $G$.
\end{theorem}
\begin{proof}[Proof of Theorem~\ref{thm:causal_reduction}.]
The reduction proceeds by invoking the good incentives oracle for each edge
$(X_i, X_j)$, taking $X_j$ as the label and using a cost function that ensures
only manipulations on $X_i$ are possible for an $\eps$-best-responding agent. If
$X_i \rightarrow X_j$, then Assumption~\ref{assump:control} ensures that improvement is
possible, and we show $\gi$ must return a classifier that incentivizes
improvement.  Otherwise, if $X_i \leftarrow X_j$, no intervention on $X_i$
can change $X_j$, so $\gi$ must return $\fail$.

More formally, let $X_i-X_j$ be an undirected edge in the skeleton $G$. We show
how to orient $X_i - X_j$ with a single oracle call. Let $X_{-j} \triangleq X
\setminus \set{X_j}$ be the set of features excluding $X_j$, and let $x_{-j}$
denote an observation of $X_{-j}$.

Consider the following good incentives problem instance. Let $X_j$ be the label,
and let the features be $(X_{-j}, \tilde{X}_i)$, where $\tilde{X}_i$ is an
identical copy of $X_i$ with structural equation $\tilde{X}_i \coloneqq X_i$.
Let the action set $\A = \mathbb{R}^n$, and let $c$ be a cost function that
ensures an $\eps$-best-responding agent will only intervene on $X_i$. In
particular, choose 
\begin{align}
    c(a; (x_{-j}, \tilde{x}_i)) = 2 B\1{a_k \neq 0 \text{ for any $k \neq i$}},
\end{align}
where $B = \sup\set{\norm{x}_\infty : x \in \X}$. In other words, the individuals pays
no cost to take actions that only affect $X_i$, but otherwise pays cost $2B$.
Since every feasible classifier $f$ takes values in $\X$, $f(x) \leq B$, and 
any action $a$ with $a_k \neq 0$ leads to negative agent utility.
At the same time, action $a=0$ has non-negative utility, so an
$\eps$-best-responding agent can only take actions that affect $X_i$.

We now show $\gi$ returns $\fail$ if and only if $X_i \gets X_j$. 
First, suppose $X_i \gets X_j$. Then $X_i$ is not a parent nor an ancestor of
$X_j$ since if there existed some $X_i \leadsto Z  \leadsto X_j$ path, then $G$
would contain a cycle. Therefore, no intervention on $X_i$ can change the
expectation of $X_j$, and consequently no classifier that can incentivize
improvement exists, so $\gi$ must return $\fail$.

On the other hand, suppose $X_i \to X_j$. We explicitly construct a classifier
$f$ that incentivizes improvement, so $\gi$ cannot return $\fail$.
By Assumption~\ref{assump:control}, there exists a function $h$ so that
\begin{align}
    \ex_{X_{-j}} \ex \brack{{X_j}_{\brack{X_i \coloneqq
    h(x_{-j})}}\paren{\set{X_{-j} = x_{-j}}}} > \ex \brack{X_j}.\label{eq:manip}
\end{align}
Since $\tilde{X}_i \coloneqq X_i$, Assumption~\ref{assump:control} still holds
additionally conditioning on $\tilde{X}_i = \tilde{x}_i$. Any classifier that
induces agents with features $(x_{-j}, \tilde{x}_i)$ to
respond by adapting only $X_i \coloneqq h(x_{-j})$ will therefore incentivize
improvement. The intervention $X_i \coloneqq h(x_{-j})$ given $X_{-j} = x_{-j}$
is incentivizable by the classifier
\begin{align}
    f((x_{-j}, \tilde{x}_i)) = \1{x_i = h(\tilde{x}_{-j})},
\end{align}
where $\tilde{x}_j$ indicates that $x_i$ is replaced by $\tilde{x}_i$ in the
vector $x_{-j}$.

An $\eps$-best-responding agent will choose action $a\opt$ where $a_i\opt =
h(\tilde{x}_{-j}) - x_i$ and otherwise $a_k\opt = 0$ in response to $f$. To see
this, $a\opt$ has cost $0$. Since $\tilde{X}_i \coloneqq X_i$, we initially have
$x_{i} = \tilde{x}_{i}$. Moreover, by construction, $h(\tilde{x}_{-j})$ depends
only on the feature copy $\tilde{x}_i$, not $x_i$, so $h(\tilde{x}_{-j})$ is
invariant to adaptations in $x_i$. Therefore, $h(\tilde{x}_{-j} + a\opt_{-i}) =
h(\tilde{x}_{-j}) = x_i + a\opt_i$, so $f((x_{-j}, \tilde{x}_i) + a\opt) = 1$.
Thus, action $a\opt$ has individual utility $1$, whereas all other actions have
zero or negative utility, so any $\eps$-best responding agent will choose
$a\opt$.  Since all agents take $a\opt$, it then follows by construction that
$\improve(f) > 0$.

Repeating this procedure for each edge in the causal graph thus fully orients
the skeleton with $|E|$ calls to $\gi$.
\end{proof}

We now turn to showing that Assumption~\ref{assump:control} holds in a large
class of nontrivial causal model, namely additive noise models
\citep{peters2017elements}.

\begin{definition}[Additive Noise Model]
    A structural causal model with graph $G = (X, E)$ is an additive noise model
    if the structural assignments are of the form
    \begin{align}
        X_j := g_j(\parents_j) + U_j \quad \text{for } j =1, \dots, n\; .
    \end{align}
    Further, we assume that all nodes $X_i$ are
    non-degenerate and that their joint distribution has a strictly positive
    density.\footnote{
        The condition that the nodes $X$ have a strictly positive
        density is met when, for example, the functional relationships $f_i$ are
        differentiable and the noise variables $U_i$ have a strictly positive
        density \citep{peters2017elements}.}
\end{definition}

Before stating the result, we need one additional technical assumption, namely
faithfulness. The faithfulness assumption is ubiquitous in causal graph
discovery setting and rules out additional conditional independence statements
that are not implied by the graph structure. For more details and a precise
statement of the d-separation criteria, see \citet{pearl2009causality}.

\begin{definition}[Faithful]
    A distribution $P_X$ is \emph{faithful} to a DAG $G$ if $A \indep B \mid C$
    implies that $A$ and $B$ are d-separated by $C$ in $G$
\end{definition}

\begin{proposition}
    \label{prop:anm}
    Let $(X_1, \dots, X_n)$ be an additive noise model, and let the joint
    distribution on $(X_1, \dots, X_n)$ be faithful to the graph $G$.
    Then, $G$ satisfies Assumption~\ref{assump:control}.
\end{proposition}
\arxiv{
\begin{proof}
For intution, we prove the result in the two-variable case and defer the full
proof to Appendix~\ref{sec:app}. Suppose $X \to Y$, so that $Y := g_Y(X) + U_Y$.
Since $X$ is non-degenerate, $X$ takes at least two values with positive
probability. Moreover, since the distribution is faithful to $G$, $g_Y$ cannot
be a constant function, since otherwise $Y \indep X$. Define $x\opt \in \argmax_x
g_Y(x)$. Then, we have
\begin{align}
    \ex_{X} \ex \brack{{Y}_{X \coloneqq x\opt}\paren{\set{X = x}}}
    = \ex_{X} \ex \brack{g_Y(x\opt) + U_Y \mid X=x}
    > \ex \brack{g_Y(X) + U_Y}
    = \ex\brack{Y}.
\end{align}
\end{proof}
}{
The proof of Proposition~\ref{prop:anm} is deferred to the appendix.
}

On the other hand, Assumption~\ref{assump:control} can indeed fail in
non-trivial cases.
\begin{example}
Consider a two variable graph with $X \to Y$. Let $Y = \eps X$ where $X$ and
$\eps$ are independent and $\ex\brack{\eps} = 0$. In general, $X$ and $Y$ are
    not independent, but for any $x, x'$, $\ex\brack{Y_{X:=x'}(\set{X=x})} = 
    x' \ex\brack{\eps} = 0 = \ex\brack{Y}$.
\end{example}

\section{Designing Good Cost Functions Requires Causal Modeling}
\label{sec:cost_reduction}
The cost function occupies a central role in the best-response agent model and
essentially determines which actions the individual undertakes. Consequently,
not few works in strategic classification model individuals as behaving
according to cost functions with desirable properties, among which is a natural
\emph{monotonicity} condition---actions that raise an individual's underlying
qualification are more expensive than those that do not.  In this section, we
prove an analogous result to the previous section and show constructing these
cost functions also requires causal modeling. 

\subsection{Outcome-Monotonic Cost Functions}
Although they use all slightly different language, \citet{milli2019social},
\citet{khajehnejad2019optimal}, and \citet{garg2019role} all assume the cost
function is well-aligned with the label. Intuitively, they both assume (i)
actions that lead to large increases in one's qualification are more costly than
actions that lead to small increases, and (ii) actions that decrease or leave
unchanged one's qualification have no cost. \citet{garg2019role} define these
cost functions using an arbitrary qualification function that maps features $X$
to label $Y$, while \citet{milli2019social} and \citet{khajehnejad2019optimal}
instead use the \emph{outcome-likelihood} $\Pr(y\mid x)$ as the qualification
function.  \citet{khajehnejad2019optimal} explicitly assume a causal
factorization so that $\Pr(y\mid x)$ is invariant to interventions on $X$, and
the qualification function of~\citet{garg2019role} ensures a similar causal
relationship between $X$ and $Y$. Translating these assumptions into the causal
framework introduced in Section~\ref{sec:framework}, we obtain a class of
\emph{outcome-monotonic} cost functions.

\begin{definition}[Outcome-monotonic cost]
    A cost function $c: \A \times \X \to \mathbb{R}_+$ is
    \emph{outcome-monotonic}
    if, for any features $x \in \X$:
    \begin{enumerate}
        \item For any action $a \in \A$, $c(a; x) = 0$ if and only if
        $\mathbb{E} \brack{Y_{X \coloneqq x + a}(\set{X = x})} \leq \mathbb{E}[Y
        \mid X=x]$.
        \item For pair of actions $a, a' \in \A$, $c(a; x) \leq c(a', x)$ if and
        only if
        \begin{align}
            \mathbb{E} \brack{Y_{X \coloneqq x + a}(\set{X = x})} \leq \mathbb{E}
        \brack{Y_{X \coloneqq x + a'}(\set{X = x})}.
        \end{align}
    \end{enumerate}
\end{definition}

While several works assume the decision-maker has access to an outcome-monotonic
cost, in general the decision-maker must explicitly construct such a cost
function from data. This challenge results in the following problem.

\begin{definition}[Learning outcome-monotonic cost problem]
    Given action set $\A$ and a joint distribution $P_{X, Y}$ over a set of
    features $X$ and label $Y$ entailed by a structural causal model, construct
    an outcome-monotonic cost function $c$.
\end{definition}

\subsection{A Reduction From Causal Modeling to Constructing Outcome-Monotonic Costs}
Outcome-monotonic costs are both conceptually
desirable~\citep{milli2019social,garg2019role} and algorithmically
tractable~\citep{khajehnejad2019optimal}. Simultaneously, outcome-monotonic cost
functions encode significant causal information, and the main result of this
section is a reduction from orienting the edges in a causal graph to learning
outcome-monotonic cost functions under the same assumption as
Section~\ref{sec:good_incentives}. Consequently, any procedure that can
successfully construct outcome-monotonic cost functions must inevitably solve a
non-trivial causal modeling problem.

\begin{proposition}
    \label{prop:cost_reduction}
    Let $G = (X, E)$ induced by a structural causal model that satisfies
    Assumption~\ref{assump:control}.  Let $\gc$ be an oracle for the
    outcome-monotonic cost learning problem.  Given the skeleton of $G$, $|E|$ calls
    to $\gc$ suffices to orient all the edges in $G$.
\end{proposition}
\begin{proof}
Let $X$ denote the variables in the causal model, and let $X_i - X_j$ be an
undirected edge. We can orient this edge with a single call
to $\gc$. Let $X_{-j} \triangleq X \setminus \set{X_j}$ denote the variables
excluding $X_j$. 

Construct an instance of the learning outcome-monotonic cost problem with features
$X_{-j}$, label $X_j$, and action set $\A = \set{\alpha e_i: \alpha \in
\mathbb{R}}$, where $e_i$ is the $i$-th standard basis vector. In other words,
the only possible actions are those that adjust the $i$-th coordinate. Let $c$
denote the outcome-monotonic cost function returned by the oracle $\gc$. We argue
$c \equiv 0$ if and only if $X_i \leftarrow X_j$.

Similar to the proof of Theorem~\ref{thm:causal_reduction}, if $X_i \leftarrow
X_j$, then $X_i$ can be neither a parent nor an ancestor of $X_j$. Therefore,
conditional on $X_{-j} = x_{-j}$, there is no intervention on $X_i$ that can
change the conditional expectation of $X_j$. Since no agent has a feasible
action that can increase the expected value of the label $X_j$ and the cost function
$c$ is outcome-monotonic, $c$ is identically $0$.

On the other hand, suppose $X_i \to X_j$. Then, by
Assumption~\ref{assump:control}, there is a real-valued function $h$ such that
\begin{align*}
    \ex_{X_{-j}} \ex\brack{{X_j}_{X_{i} \coloneqq h(x_{-j})}\paren{\set{X_{-j}
    = x_{-j}}}} > \ex\brack{X_j}.
\end{align*}
This inequality along with the tower property then implies there is some
agent $x_{-j}$ such that
\begin{align*}
    \ex\brack{{X_j}_{X_{i} \coloneqq h(x_{-j})}\paren{\set{X_{-j} = x_{-j}}}} >
\ex\brack{X_j \mid X_{-j} = x_{-j}}, 
\end{align*}
since otherwise the expectation would be zero or negative.  Since $h(x_{-j})e_i
\in \A$ by construction, there is some action $a \in \A$ that can increase the
expectation of the label $X_j$ for agents with features $x_{-j}$, so $c(a;
x_{-j}) \neq 0$, as required.
\end{proof}

The proof of Proposition~\ref{prop:cost_reduction} makes repeated calls to an
oracle to construct outcome-monotonic cost functions to decode the causal
structure of the graph $G$. In many cases, however, even a single
outcome-monotonic cost function encode significant information about the
underlying graph, as the following example shows.

\begin{example}
Consider a causal model with features $(X, Z)$ and label $Y$ with the following
    structural equations
\begin{align}
    \label{eq:linear_model}
    X_i &\coloneqq U_{X_i} \quad \text{for } i = 1, \dots, n\\
    Y &\coloneqq \sum_{i=1}^n \theta_i X_i + U_Y \\
    Z_j &\coloneqq g_j(X, Y, U_{Z_j}) \quad \text{for } j = 1, \dots, m,\\
\end{align}
for some set of non-zero coefficients $\theta_i \in \mathbb{R}$ and arbitrary
functions $g_j$. In other words, the model consists of $n$ causal features,
$m$ non-causal features, and a linear structural equation for $Y$.

Suppose the action set $\A = \mathbb{R}^{n + m}$, and let $c$ be any
outcome-monotonic cost. Then, $2(n + m)$ queries evaluations of $c$
suffice to determine (1) which features are causal, and (2)
$\sign(\theta_i)$ for $i=1, \dots, n$. To see this,  evaluate the cost
function at points $c(e_i; 0)$ and $c(-e_i; 0)$, where $e_i$ denotes the
$i$-th standard basis vector. Direct
calculation shows 
\begin{align*}
    \mathbb{E} \brack{Y_{(X, Z) \coloneqq e_i}(\set{(X, Z) = 0})}
    =
    \begin{cases}
        \theta_i &\;\;\text{if feature $i$ is causal} \\
        0 &\;\;\text{otherwise.}
    \end{cases}
\end{align*}
Therefore, since $c$ is outcome-monotonic, if $c(e_i; 0) > 0$, then $\sign(\theta_i) =
1$, if $c(-e_i; 0) > 0$, then $\sign(\theta_i) = -1$, and if both $c(e_i; 0) =
0$ and $c(-e_i; 0) = 0$, then feature $i$ is non-causal. 
\end{example}

\section{Discussion}
\label{sec:conclusion}
The large collection of empirical examples of failed incentive schemes is a
testament to the difficulty of designing incentives for individual improvement.
In this work, we argued an important source of this difficulty is that
incentivize design must inevitably grapple with causal analysis.  Our results
are not hardness results per se. There are no fundamental computational or
statistical barriers that prevent causal modeling beyond the standard
unidentifiability results in causal inference. Rather, our work suggests
attempts to design incentives for improvement without some sort of causal
reasoning are unlikely to succeed.

Beyond incentive design, we hope our causal perspective clarifies intuitive,
though subtle notions like gaming and improvement and provides a clear and
consistent formalism for reasoning about strategic adaptation more broadly.


\bibliographystyle{plainnat}
\bibliography{refs}

\clearpage
\newpage
\onecolumn
\appendix
\section{Missing Proofs}
\label{sec:app}
\begin{proof}[Proposition~\ref{prop:anm}]
Let $V \to W$ be an edge in $G$. We show there exists a real-valued function $h$
that maps a realization of nodes $X_{-W} = x_{-w}$ to an intervention $v\opt$
that increases the expected value of $W$. Therefore, we first condition on
observing the remaining nodes $X_{-W} = x_{-w}$. In an additive noise model,
given $X_{-W} = x_{-w}$ the exogenous noise terms for all of the ancestors of
$W$ can be uniquely recovered. In particular, the noise terms are determined by
\begin{align}
    u_j = x_j - g_j(\parents_j).
\end{align}
Let $U_{\ancestors}$ denote the collection of noise variables for ancestors of
$W$ \emph{excluding} those only have a path through $V$. Both $U_{\ancestors} =
u_{\ancestors}$ and $V=v$ are fixed by $X_{-W} = x_{-w}$. 

Consider the structural equation for $W$, $W = g_W(\parents_{W}) + U_W$.
The parents of $W$, $\parents_{W}$, are deterministic given $V$ and $U_{\ancestors}$. 
Therefore, given $V=v$ and $U_\ancestors=u_\ancestors$, $g_W(\parents_{W})$ is a
deterministic function of $v$ and $u_\ancestors$, which we write $\tilde{g}_W(v,
u_{\ancestors})$.

Now, we argue $\tilde{g}_W$ is not constant in $v$. Suppose $\tilde{g}_W$ were
constant in $v$. Then, for every $u_\ancestors$, $\tilde{g}_W(v, u_\ancestors) =
k(u_\ancestors)$. However, this means $W = k(U_\ancestors) + U_W$, and
$U_\ancestors$ is independent of $V$, so we find that $V$ and $W$ are
independent. However, since $V \to W$ in $G$, this contradicts faithfulness.

Since $\tilde{g}_W$ is not constant in $v$, there exists at least one setting of
$u_\ancestors$ with $v, v'$ so that $\tilde{g}_W(v', u_\ancestors) > \tilde{g}_W(v,
u_\ancestors)$. Since $X$ has positive density, $(v, u_a)$ occurs with positive
probability. Consequently, if $h(u_\ancestors) = \argmax_{v} \tilde{g}_W(v,
u_\ancestors)$, then
\begin{align}
    \ex_{X_{-W}} \ex\brack{W_{V\coloneqq v\opt(u_\ancestors)}\paren{\set{X_{-W} = x_{-w}}}}
    &= \ex_{X_{-W}} \brack{\ex\brack{U_W} + \ex\brack{\tilde{g}_{W}(v\opt(U_A), U_A) \mid X_{-W} = x_{-w}}} \\
    &> \ex\brack{U_W} + \ex_{X_{-W}}\ex\brack{\tilde{g}_{W}(V, U_A) \mid X_{-W} = x_{-w}}  \\
    &= \ex\brack{U_W} + \ex\brack{g_{W}(\parents_W)} \\
    &= \ex\brack{W}.
\end{align}
Finally, notice $h(u_\ancestors)$ can be computed solely form $x_{-w}$ since
$u_{\ancestors}$ is fixed given $x_{-w}$. Together, this establishes that
Assumption~\ref{assump:control} is satisfied for the additive noise model.
\end{proof}

\end{document}